\documentclass{article}

\usepackage[preprint]{neurips_2026}

\usepackage[utf8]{inputenc}
\usepackage[T1]{fontenc}
\usepackage{hyperref}
\usepackage{url}
\usepackage{booktabs}
\usepackage{amsfonts}
\usepackage{amsmath}
\usepackage{amsthm}
\usepackage{amssymb}
\usepackage{nicefrac}
\usepackage{microtype}  
\usepackage{xcolor}
\usepackage{graphicx}
\usepackage{subcaption}
\usepackage{multirow}
\usepackage{float}

\newcommand{\bfb}{\boldsymbol{\beta}}
\newcommand{\bfeta}{\boldsymbol{\eta}}
\newcommand{\x}{\mathbf{x}}
\newcommand{\X}{\mathbf{X}}
\newcommand{\y}{\mathbf{y}}
\newcommand{\I}{\mathbf{I}}
\newcommand{\bP}{\mathbf{P}}

\newtheorem{theorem}{Theorem}
\newtheorem{corollary}{Corollary}
\newtheorem{remark}{Remark}

\DeclareMathOperator{\argmax}{argmax}

\title{TinyBayes: Closed-Form Bayesian Inference via Jacobi Prior for Real-Time Image Classification on Edge Devices}

\author{
  Shouvik Sardar\thanks{Correspondence to \texttt{sardar@cmi.ac.in}} \\
  Data Science Group\\
  Chennai Mathematical Institute\\
  ORCID: \href{https://orcid.org/0009-0004-0550-0271}{0009-0004-0550-0271}
  \And
  Sourish Das \\
  Data Science Group\\
  Chennai Mathematical Institute\\
  ORCID: \href{https://orcid.org/0000-0002-5354-6520}{0000-0002-5354-6520}
}

\begin{document}

\maketitle

\begin{abstract}
Cocoa (\textit{Theobroma cacao}) is a critical cash crop for millions of smallholder farmers in West Africa, where Cocoa Swollen Shoot Virus Disease (CSSVD) and anthracnose cause devastating yield losses. Automated disease detection from leaf images is essential for early intervention, yet deploying such systems in resource-constrained settings demands models that are small, fast, and require no internet connectivity. Existing edge-deployable plant disease systems rely on end-to-end deep learning without uncertainty quantification, while Bayesian methods for edge devices focus on hardware-level inference architectures rather than agricultural applications. We bridge this gap with TinyBayes, the first framework to combine a closed-form Bayesian classifier with a mobile-grade computer vision pipeline for crop disease detection. Our pipeline uses YOLOv8-Nano (5.9\,MB) for lesion localisation, MobileNetV3-Small (3.5\,MB) for feature extraction, and the Jacobi prior; a Bayesian method that provides exact, non-iterative estimators via projection; for classification. The Jacobi-DMR (Distributed Multinomial Regression) classifier adds only 13.5\,KB to the pipeline, bringing the total model size within 9.5\,MB, while achieving 78.7\% accuracy on the Amini Cocoa Contamination Challenge dataset and enabling end-to-end CPU inference under 150\,ms per image. We benchmark against seven classifiers including Random Forest, SVM, Ridge, Lasso, Elastic Net, XGBoost, and Jacobi-GP, and demonstrate that the Jacobi-DMR offers the best trade-off between accuracy, model size, and inference speed for edge deployment. We have proved the asymptotic equivalence and consistency, asymptotic normality and the bias correction of Jacobi-DMR. All data and codes are available here: \url{https://github.com/shouvik-sardar/TinyBayes}
\end{abstract}

\section{Introduction}

Cocoa (\textit{Theobroma cacao} L.) underpins a global industry worth over \$130 billion annually, with more than 60\% of production concentrated in West Africa~\citep{ICCO2023}. Smallholder farmers, who account for the vast majority of cocoa cultivation, face persistent threats from plant diseases. Cocoa Swollen Shoot Virus Disease (CSSVD), transmitted by mealybugs, causes leaf swelling, shoot dieback, and eventual tree death with no known cure~\citep{Dzahini2006}. Anthracnose, caused by \textit{Colletotrichum} species, produces necrotic lesions that reduce both yield and bean quality~\citep{Marelli2019}. Traditional identification by trained extension officers is slow, expensive, and inaccessible in remote regions.

Recent advances in computer vision have enabled automated plant disease detection~\citep{Mohanty2016}. However, deploying deep learning models in Sub-Saharan Africa faces three practical constraints: (i)~limited computational resources on farmers' entry-level smartphones, (iii)~very patchy to no-internet connectivity; and (ii)~the need for uncertainty quantification to support reliable decision-making. These constraints define three active but largely disjoint research threads.

The first thread, \emph{TinyML for agriculture}, applies lightweight architectures such as MobileNet and quantised models to on-device plant disease classification~\citep{Malche_Joshi_Upadhyay_Soni}. These systems are computationally efficient but lack any Bayesian component, offering point predictions without uncertainty estimates.

The second thread, \emph{Bayesian inference on edge hardware}, develops specialised architectures; memristors, ferroelectric transistors, compute-in-memory chips; for probabilistic computation at the hardware level. These approaches remain in the hardware research domain and are not applicable to software deployment on commodity smartphones.

The third thread, \emph{Bayesian deep learning}, integrates uncertainty into neural networks via MC-Dropout~\citep{Gal2016Dropout}, Bayes-by-Backprop~\citep{Blundell2015BayesByBackprop}, or variational inference~\citep{Blei2017VI}. These methods require modified training procedures and iterative inference, increasing computational cost beyond what edge devices can accommodate.

\textbf{Our contribution sits at the intersection of these three threads.} To our knowledge, TinyBayes is the first framework to bring closed-form Bayesian inference; with principled uncertainty quantification and sub-kilobyte model size; to edge-deployable crop disease classification. We apply the Jacobi prior~\citep{Das_Dey_2006, Das2008, Das_Sardar_2025}, which places a conjugate prior on the canonical parameter of the exponential family likelihood and yields a closed form, non-iterative posterior mode estimator via projection. Combined with YOLOv8-Nano for detection and MobileNetV3-Small for feature extraction, the complete pipeline fits in 9.5\,MB and runs in 150\,ms on a single CPU core.


\begin{figure}[H]
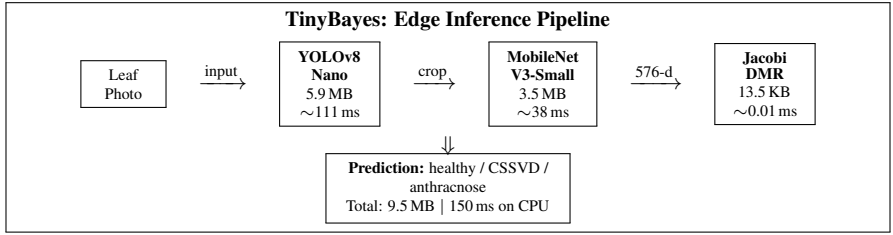

\centering
\footnotesize
\setlength{\fboxsep}{3pt}
\fbox{\parbox{0.82\textwidth}{
\centering
\textbf{TinyBayes: Edge Inference Pipeline}\\[0.4em]
\begin{tabular}{ccccccc}
\fbox{\parbox{0.9cm}{\centering\tiny Leaf\\Photo}} &
$\xrightarrow{\text{\tiny input}}$ &
\fbox{\parbox{1.1cm}{\centering\tiny\textbf{YOLOv8}\\\textbf{Nano}\\5.9\,MB\\$\sim$111\,ms}} &
$\xrightarrow{\text{\tiny crop}}$ &
\fbox{\parbox{1.2cm}{\centering\tiny\textbf{MobileNet}\\\textbf{V3-Small}\\3.5\,MB\\$\sim$38\,ms}} &
$\xrightarrow{\text{\tiny 576-d}}$ &
\fbox{\parbox{1.1cm}{\centering\tiny\textbf{Jacobi}\\[-0.2em]\textbf{DMR}\\13.5\,KB\\$\sim$0.01\,ms}}
\end{tabular}\\[0.3em]
$\Downarrow$\\[-0.2em]
\fbox{\parbox{3.0cm}{\centering\tiny\textbf{Prediction:} healthy / CSSVD / anthracnose\\Total: 9.5\,MB $|$ 150\,ms on CPU}}
}}
\caption{TinyBayes end-to-end inference pipeline: component sizes, latencies, and total CPU time.}
\label{fig:pipeline}
\end{figure}

Our contributions are:
\begin{itemize}
    \item We present the first edge-deployable Bayesian pipeline for crop disease classification, combining YOLOv8-Nano (5.9\,MB), MobileNetV3-Small (3.5\,MB), and Jacobi-DMR (13.5\,KB) for a total of 9.5\,MB.
    \item We benchmark the Jacobi-DMR against seven classifiers on 576-dimensional MobileNetV3 features, showing competitive accuracy (78.7\%) with the fastest training (0.06\,s) and smallest model size.
    \item We demonstrate that the Jacobi-GP variant achieves the highest accuracy (86\%) among all methods, serving as a theoretical upper bound for the Jacobi framework's capacity.
    \item We provide a PCA ablation study across all eight classifiers, showing that dimensionality reduction via PCA does not offer any significant benefit while adding 450\,KB of storage overhead; validating the use of all 576 raw features.
    \item We proved asymptotic consistency, asymptotic normality, bias correction and a finite-sample hyperparameter invariance property of the Jacobi-DMR estimator in Section~\ref{sec:th_results}.
    
\end{itemize}

\section{The Jacobi Prior}\label{sec:jacobi}

Here we present the Jacobi prior framework briefly; while full details are in~\citet{Das_Dey_2006, Das2008, Das_Sardar_2025}. It presents the closed form analytical estimator for the regression coefficient for $\bfb$. 

\paragraph{Generalised Linear Model Setup}
Consider data $\mathcal{D} = \{(y_i, \x_i) \mid i = 1,\ldots,n\}$ where $y_i$ follows an exponential family distribution with canonical parameter $\theta_i$, link function $g(\cdot)$, and linear predictor $\eta_i = \x_i'\bfb$. Standard Bayesian inference places a prior $\pi(\bfb)$ on the regression coefficients, yielding a posterior that requires expensive high-dimensional integration or iterative optimisation.

\paragraph{Prior on Canonical Parameters}
The Jacobi prior instead assigns a conjugate prior $\pi(\theta_i)$ to each canonical parameter. Since $g(\cdot)$ is one-to-one with Jacobian $J = |g'(\eta_i)|$, the posterior of $\eta_i$ is $q(\eta_i|y_i) = \pi(g(\eta_i)|y_i) \cdot J$, and its mode $\hat{\eta}_i = h(y_i)$ is available in closed form.

\paragraph{Jacobi Estimator via Projection}
The vector $\hat{\bfeta} = (h(y_1),\ldots,h(y_n))'$ is projected onto the column space of $\X$:
\begin{equation}\label{eq:jacobi_est}
    \hat{\bfb} = (\X'\X)^{-1}\X'\hat{\bfeta}.
\end{equation}
This is a least-squares projection that decomposes $h(\y) = \bP h(\y) + (\I - \bP)h(\y)$, where $\bP = \X(\X'\X)^{-1}\X'$ is the linear transformation matrix representing the orthogonal projection from $n$-dimensional space $\mathbb{R}^n$ onto the estimation space, while $(\mathbf{I} - \mathbf{P})$ represents the orthogonal projection of $\mathbb{R}^n$ onto the error space.

\paragraph{Jacobi Estimator for Distributed Multinomial Regression (DMR)}
Following~\citet{Taddy_2015}, multinomial classification with $K$ classes is handled via a Poisson surrogate representation. Although the observed responses satisfy $Y_{ik} \in \{0,1\}$ with $\sum_{k=1}^K Y_{ik}=1$, the DMR formulation treats each class $k$ as arising from a Poisson regression with rate $\lambda_{ik}$. This is a working model; under a log-linear parameterisation, the Poisson likelihood yields the same score equations as the multinomial model \citep{Taddy_2015}, leading to consistent estimation of the multinomial parameters despite the distributional mismatch. Placing the conjugate prior $\lambda_{ik} \sim \mathrm{Gamma}(a_n, b_n)$ (rate parametrisation) under this working model, the Jacobi posterior mode for each class $k$ is $\hat{\eta}_{ik} = \ln\!\big(\frac{y_{ik}+a}{1+kb}\big)$, and the class-specific coefficient vector is $\hat{\bfb}_k = (\X'\X)^{-1}\X'\hat{\bfeta}_k$. Prediction assigns $\argmax_k \exp(\x_{\text{new}}'\hat{\bfb}_k)$.

\paragraph{Gaussian Process Extension}
For nonlinear boundaries, the linear predictor is augmented with a GP random effect: $\eta_i = \x_i'\bfb + w_i + \varepsilon_i$, where $w_i \sim \mathcal{GP}(0, \Sigma)$ with squared exponential kernel. The Jacobi-estimated $\hat{\bfeta}$ serves as the GP's target, providing a closed-form initialisation for the nonparametric component.

\section{Theoretical Results}\label{sec:th_results}
We extend the asymptotic theory of the Jacobi estimator for Generalised Linear Models \citep[Theorem~2.3]{Das_Sardar_2025} to the Jacobi-DMR  of Section~\ref{sec:jacobi}.

\paragraph{Model and estimator.}
Following \citet{Taddy_2015} and \citet[Section~3.4]{Das_Sardar_2025}, under the
DMR Poisson surrogate, the multinomial classification problem with $K$ classes
is decomposed into $K$ independent Poisson regressions:
\[
  Y_{ik} \;\sim\; \mathrm{Poisson}\!\left( \lambda_{0,ik} \right),
  \qquad
  \eta_{0,ik} \;=\; \log \lambda_{0,ik} \;=\; \mathbf{x}_{i}^{\!\top}\boldsymbol{\beta}_{0,k},
  \qquad
  i = 1, \dots, n,\; k = 1, \dots, K.
\]
Placing the conjugate prior $\lambda_{ik} \sim \mathrm{Gamma}(a_{n}, b_{n})$
(rate parametrisation) on each $\lambda_{ik}$, the closed-form posterior mode
of $\eta_{ik}$ is
\[
  \hat{\eta}_{ik}
  \;=\; \log\!\frac{Y_{ik} + a_{n}}{1 + b_{n}},
\qquad
  \hat{\boldsymbol{\beta}}_{k}^{(n)}
  \;=\; (\mathbf{X}^{\!\top}\mathbf{X})^{-1}\mathbf{X}^{\!\top}\hat{\boldsymbol{\eta}}_{k},
\qquad
  \hat{\boldsymbol{\eta}}_{k}
   = (\hat{\eta}_{1k}, \dots, \hat{\eta}_{nk})^{\!\top}.
\]

Let $\tilde{\boldsymbol{\beta}}_{k}^{(n)}$ denote the posterior mode of
$\boldsymbol{\beta}_{k}$ under the induced Jacobi prior $\pi_{J}^{(k)}$.
Define the stacked Jacobi-DMR estimator as
$\hat{\boldsymbol{\beta}}^{(n)}
= \big( (\hat{\boldsymbol{\beta}}_{1}^{(n)})^{\!\top}, \dots, (\hat{\boldsymbol{\beta}}_{K}^{(n)})^{\!\top} \big)^{\!\top}
\in \mathbb{R}^{Kp}$,
and analogously define $\tilde{\boldsymbol{\beta}}^{(n)}$ and $\boldsymbol{\beta}_{0}$.

\paragraph{Assumptions.}
For each $k \in \{1, \dots, K\}$:
\begin{itemize}
  \item[(B1)] The Poisson log-likelihood $\ell_{k}(y, \eta) = y\,\eta - e^{\eta}$
        is twice continuously differentiable in $\eta$, and the design matrix
        $\mathbf{X} \in \mathbb{R}^{n \times p}$ has full column rank $p$.
  \item[(B2)] The induced Jacobi prior $\pi_{J}^{(k)}(\boldsymbol{\beta}_{k})$ on
        $\boldsymbol{\beta}_{k}$ and the prior $\pi_{n}^{(k)}(\boldsymbol{\eta}_{k})$
        on $\boldsymbol{\eta}_{k}$ are continuous and strictly positive on
        $\varepsilon$-balls around $\boldsymbol{\beta}_{0,k}$ and
        $\boldsymbol{\eta}_{0,k} = \mathbf{X}\boldsymbol{\beta}_{0,k}$, respectively.
  \item[(B3)] The normalised log-likelihood satisfies a uniform law of large
        numbers,
        \[
          \sup_{\boldsymbol{\beta}_{k} \in \mathcal{B}}
            \bigg|\, \frac{1}{n}\sum_{i=1}^{n}
                  \ell_{k}\!\left(Y_{ik},\, \mathbf{x}_{i}^{\!\top}\boldsymbol{\beta}_{k}\right)
              - \mathbb{E}\!\left[\ell_{k}\!\left(Y_{1k},\, \mathbf{x}_{1}^{\!\top}\boldsymbol{\beta}_{k}\right)\right]
            \bigg|
          \;\xrightarrow{P}\; 0.
        \]
  \item[(B4)] $Q_{k}(\boldsymbol{\beta}_{k}) := \mathbb{E}\!\left[\ell_{k}\!\left(Y_{1k},\, \mathbf{x}_{1}^{\!\top}\boldsymbol{\beta}_{k}\right)\right]$
        admits a unique maximiser $\boldsymbol{\beta}_{0,k} \in \mathbb{R}^{p}$,
        and $\boldsymbol{\eta}_{0,k} = \mathbf{X}\boldsymbol{\beta}_{0,k}$
        uniquely maximises
        $R_{k}(\boldsymbol{\eta}_{k}) := \mathbb{E}\!\left[\ell_{k}\!\left(Y_{1k},\, \eta_{1k}\right)\right]$.
  \item[(B5)] The hyperparameters scale as $a_{n} = b_{n} = 1/n$.
  \item[(B6)] (\emph{Taddy DMR Poisson surrogate.})
        $\{Y_{ik}\}_{i, k}$ are mutually independent across $i$ and across $k$,
        with $Y_{ik} \sim \mathrm{Poisson}\!\left(\exp(\mathbf{x}_{i}^{\!\top}\boldsymbol{\beta}_{0,k})\right)$. \\
\end{itemize}

\begin{theorem}[Asymptotic Equivalence and Consistency]
\label{thm:dmr-equiv}
Under Assumptions (B1)--(B6):
\begin{enumerate}
  \item[\emph{(i)}] For each $k = 1, \dots, K$, $\hat{\boldsymbol{\beta}}_{k}^{(n)} \;\xrightarrow{P}\; \boldsymbol{\beta}_{0,k}, \quad \tilde{\boldsymbol{\beta}}_{k}^{(n)} \;\xrightarrow{P}\; \boldsymbol{\beta}_{0,k} \quad (n \to \infty).$
  \item[\emph{(ii)}] For each $k$ and every $\varepsilon, \delta > 0$ there exists an integer $N_{k}(\varepsilon, \delta)$ such that for all $n \ge N_{k}(\varepsilon, \delta)$, $P\!\left( \big\| \hat{\boldsymbol{\beta}}_{k}^{(n)} - \tilde{\boldsymbol{\beta}}_{k}^{(n)} \big\| > \varepsilon \right) \;<\; \delta.$
  \item[\emph{(iii)}] Jointly, $\hat{\boldsymbol{\beta}}^{(n)} \xrightarrow{P} \boldsymbol{\beta}_{0}$ in $\mathbb{R}^{Kp}$, and for every $\varepsilon, \delta > 0$ there exists an integer $N(\varepsilon, \delta)$ such that for all $n \ge N(\varepsilon, \delta)$, $P\!\left( \big\| \hat{\boldsymbol{\beta}}^{(n)} - \tilde{\boldsymbol{\beta}}^{(n)} \big\| > \varepsilon \right) \;<\; \delta.$ \hfill Proof in Appendix~\ref{sec:appendix}

\end{enumerate}
\end{theorem}

\begin{corollary}[Hyperparameter Invariance]
\label{cor:invariance}
Suppose the targets are one-hot encoded: $Y_{ik} \in \{0, 1\}$ for all $i, k$
with $\sum_{k = 1}^{K} Y_{ik} = 1$. For any $a > 0$ and $b > 0$, the
Jacobi-DMR predicted class
\begin{eqnarray*}
  \hat{C}(\mathbf{x};\, a, b)
  \;:=\; \arg\max_{k \in \{1, \dots, K\}}\,
         \mathbf{x}^{\!\top}\hat{\boldsymbol{\beta}}_{k}(a, b), where\\
\qquad
  \hat{\boldsymbol{\beta}}_{k}(a, b)
   = (\mathbf{X}^{\!\top}\mathbf{X})^{-1}\mathbf{X}^{\!\top}\hat{\boldsymbol{\eta}}_{k}(a, b),
\qquad
  \hat{\eta}_{ik}(a, b) = \log\!\frac{Y_{ik} + a}{1 + b},
\end{eqnarray*}
is invariant to $(a, b)$ as the transformation induces a class(k)-independent shift and a common positive scaling of class scores, preserving the argmax decision intact. \hfill Appendix~\ref{sec:appendix}.\\

\end{corollary}

\begin{remark}
\label{rem:invariance-vs-consistency}
Corollary~\ref{cor:invariance} is a finite-sample, distribution-free property;
which is
consistent with Theorem~\ref{thm:dmr-equiv}: the coefficient vector
$\hat{\boldsymbol{\beta}}_{k}(a, b)$ does depend on $(a, b)$, and only the consistent choice $a_{n} = b_{n} = 1/n$
delivers $\hat{\boldsymbol{\beta}}_{k}^{(n)} \xrightarrow{P} \boldsymbol{\beta}_{0,k}$.
The argmax decision, however, is invariant in $(a, b)$ for the Jacobi DMR framework. \hfill Appendix~\ref{sec:appendix}.\\
\end{remark}

\begin{remark}[Empirical Validation]
\label{rem:empirical-invariance}
We evaluated Jacobi-DMR over a grid of $(a,b) \in \{1/2000, 0.01, 0.05, 0.1, 0.2, 0.5, 1.0\}^2$
using all 576 features and found classification accuracy is invariant across the
entire grid at 78.7\%, consistent with Corollary~\ref{cor:invariance}. \hfill Appendix~\ref{sec:appendix}.
\end{remark}

\paragraph{Asymptotic normality under a high-intensity regime.}
To analyse the distribution of the Jacobi-DMR estimator, we introduce a high-intensity scaling of the Poisson surrogate. Specifically, we consider a sequence of models in which the Poisson rates are rescaled as $\lambda_{0,ik}^{(n)} = s_n \exp(\mathbf{x}_i^\top\boldsymbol{\beta}_{0,k})$, with $s_n \to \infty$. This scaling can be interpreted as an aggregation or exposure regime, where each observation represents $s_n$ replicated events. Under this regime, the normalised fluctuation $\frac{Y_{ik} - \lambda_{0,ik}^{(n)}}{\lambda_{0,ik}^{(n)}}$ vanishes in probability, ensuring that the log-transformed responses admit a second-order linear approximation. This linearisation is the key step enabling a central limit theorem for the projected estimator. It is important to note that this regime is not intrinsic to one-hot classification data, where $Y_{ik} \in \{0,1\}$ and the implied rates are bounded. Rather, it serves as a theoretical device for characterising the estimator’s asymptotic behaviour. The practically relevant bounded-rate setting is addressed separately via bias correction in Remark~\ref{rem:bias-correction}.

\paragraph{Additional assumptions.}
\begin{itemize}
  \item[(C1)] (\emph{Large-rate regime.}) There exists a sequence
        $\lambda_{\min, n} \to \infty$ with $\lambda_{\min, n} = \omega(\sqrt{n})$
        such that
        \[
          \inf_{1 \le i \le n,\; 1 \le k \le K}\, \lambda_{0, i k}
          \;=\; \inf_{i, k}\, \exp(\mathbf{x}_{i}^{\!\top}\boldsymbol{\beta}_{0,k})
          \;\ge\; \lambda_{\min, n}
\qquad \text{for all sufficiently large } n.
        \]
  \item[(C2)] (\emph{Hyperparameter rate.}) $a_{n}, b_{n} \to 0$ with
        $\sqrt{n}\,(a_{n} + b_{n}) \to 0$. In particular, $a_{n} = b_{n} = 1/n$
        satisfies this.
  \item[(C3)] (\emph{Design Gram limit.}) $n^{-1}\mathbf{X}^{\!\top}\mathbf{X} \to \mathbf{Q}$
        for some positive-definite $\mathbf{Q} \in \mathbb{R}^{p \times p}$.
  \item[(C4)] (\emph{Lindeberg condition.}) $\max_{1 \le i \le n} \|\mathbf{x}_{i}\|^{2} / n \to 0$,
        and the triangular array
        $\{\mathbf{x}_{i}\,(Y_{ik} - \lambda_{0,ik}) / \lambda_{0,ik}\}_{i = 1}^{n}$
        satisfies the Lindeberg condition for each $k$. (A sufficient condition
        is $\sup_{i}\|\mathbf{x}_{i}\| < \infty$.) \\
\end{itemize}

\begin{theorem}[Asymptotic Normality]
\label{thm:dmr-clt}
Under Assumptions (B1)--(B6) and (C1)--(C4), for each $k = 1, \dots, K$,
$\sqrt{n}\,\big( \hat{\boldsymbol{\beta}}_{k}^{(n)} - \boldsymbol{\beta}_{0,k} \big)
\xrightarrow{d}
\mathcal{N}\!\big( \mathbf{0},\; \mathbf{Q}^{-1}\,\mathbf{V}_{k}\,\mathbf{Q}^{-1} \big)$,
where $\mathbf{V}_{k}
:= \lim_{n \to \infty}
\frac{1}{n}\sum_{i=1}^{n} \mathbf{x}_{i}\mathbf{x}_{i}^{\!\top}
e^{-\mathbf{x}_{i}^{\!\top}\boldsymbol{\beta}_{0,k}}$.
For the stacked estimator $\hat{\boldsymbol{\beta}}^{(n)} \in \mathbb{R}^{Kp}$,
$\sqrt{n}\,\big( \hat{\boldsymbol{\beta}}^{(n)} - \boldsymbol{\beta}_{0} \big)
\xrightarrow{d}
\mathcal{N}\!\big( \mathbf{0},\; \boldsymbol{\Sigma} \big)$,
where $\boldsymbol{\Sigma} \in \mathbb{R}^{Kp \times Kp}$ is the block-diagonal
matrix with $k$-th diagonal block $\mathbf{Q}^{-1}\mathbf{V}_{k}\mathbf{Q}^{-1}$.
The block-diagonal structure follows from the cross-class independence in
Assumption (B6). \hfill Appendix~\ref{sec:appendix}.\\
\end{theorem}

\begin{remark}[Bias correction in the moderate-rate regime]
\label{rem:bias-correction}
When the large-rate condition (C1) is replaced by the \emph{moderate-rate} condition $\inf_{i, k}\,\lambda_{0,ik} \ge \lambda_{\min} > 0$ (rates bounded
away from zero but not necessarily divergent), the Anscombe-type bias
$-1/(2\,\lambda_{0,ik})$ no longer vanishes asymptotically, and the centring
in Theorem~\ref{thm:dmr-clt} requires a correction. Define the
\emph{bias-corrected} Jacobi-DMR estimator
\[
  \hat{\boldsymbol{\beta}}_{k}^{(n,\,\mathrm{bc})}
  \;:=\; \hat{\boldsymbol{\beta}}_{k}^{(n)}
        \;+\; \frac{1}{2}\,(\mathbf{X}^{\!\top}\mathbf{X})^{-1}\mathbf{X}^{\!\top}
              \Big( e^{-\mathbf{x}_{i}^{\!\top}\hat{\boldsymbol{\beta}}_{k}^{(n)}} \Big)_{i = 1}^{n}.
\]
Under Assumptions (B1)--(B6), (C2)--(C4) and the moderate-rate condition, an
analogous Taylor argument together with the consistency
$\hat{\boldsymbol{\beta}}_{k}^{(n)} \xrightarrow{P} \boldsymbol{\beta}_{0,k}$
(Theorem~\ref{thm:dmr-equiv}) yields
\[
  \sqrt{n}\,\big( \hat{\boldsymbol{\beta}}_{k}^{(n,\,\mathrm{bc})}
                  - \boldsymbol{\beta}_{0,k} \big)
  \;\xrightarrow{d}\;
  \mathcal{N}\!\big( \mathbf{0},\; \mathbf{Q}^{-1}\mathbf{V}_{k}\mathbf{Q}^{-1} \big).
\]
The correction is an explicit, closed-form post-processing of
$\hat{\boldsymbol{\beta}}_{k}^{(n)}$ requiring no additional iterative
computation, therefore preserves the edge-deployability of the Jacobi-DMR
pipeline. \hfill Appendix~\ref{sec:appendix}.
\end{remark}

\section{Dataset}\label{sec:data}
The Amini Cocoa Contamination Challenge~\citep{AminiCocoa2024} provides annotated cocoa leaf images from West African farms with three class labels: \textbf{healthy}, \textbf{CSSVD} (chlorotic patches and vein banding from viral infection), and \textbf{anthracnose} (necrotic lesions from \textit{Colletotrichum gloeosporioides}). Via stratified sampling with fixed random seeds, we constructed a \textbf{training set} of 2{,}000 balanced images (667/667/666 per class) and a disjoint \textbf{validation set} of 500 images (167/166/167 per class).

\section{Methodology for Edge Deployment}\label{sec:method}

Our pipeline has three stages, each designed for edge deployment. Figure~\ref{fig:pipeline} illustrates the complete inference flow.

\paragraph{Stage 1: Lesion Localisation via YOLOv8-Nano.} We trained YOLOv8-Nano~\citep{Jocher2023YOLO}, a 5.9\,MB object detector designed for mobile devices, on the training set bounding box annotations for 10 epochs (batch size 16, image size $640\times640$, Adam optimiser). At inference, boxes with confidence $>0.25$ are retained. This replaces the Faster R-CNN used in prior work~\citep{Das_Sardar_2025}.

\paragraph{Stage 2: Feature Extraction via MobileNetV3-Small.} A pretrained MobileNetV3-Small~\citep{Howard2019MobileNetV3} (3.5\,MB, ImageNet weights) serves as a fixed feature extractor. The classifier head is replaced with an identity mapping, yielding 576-dimensional feature vectors from each bounding box crop resized to $224\times224$. All 576 features are retained without dimensionality reduction.

\paragraph{Stage 3: Classification.} We evaluate eight classifiers on the extracted features: \textbf{Random Forest (RF)} (100 trees, all 576 features); \textbf{SVM} (RBF kernel, $C=1.0$, $\gamma=\text{scale}$, standardised features); \textbf{Ridge} ($\alpha=1.0$, standardised features); \textbf{Lasso} (L1-penalised logistic regression, SAGA solver, $C=1.0$); \textbf{Elastic Net} (L1 ratio $=0.5$, SAGA solver, $C=1.0$); \textbf{XGBoost} (100 trees, max depth 6, learning rate 0.1); \textbf{Jacobi-DMR} (Distributed Multinomial Regression with $a=b=1/n=1/2000$, the consistent choice per~\citet{Das_Sardar_2025}); \textbf{Jacobi-GP} (Gaussian Process regression on Jacobi-estimated targets, RBF kernel, $a=b=1/2000$).

\section{Empirical Results}\label{sec:results}

\subsection{Predictive Accuracy}

\paragraph{Classification Performance}

Table~\ref{tab:th_results} presents the comprehensive results for all eight classifiers on 576-dimensional MobileNetV3 features.

\begin{table}[H]
\centering
\caption{Classification performance on the validation set. E2E Pred includes YOLOv8-Nano detection + MobileNetV3 feature extraction + classifier prediction, measured on CPU.}
\label{tab:th_results}
\small
\begin{tabular}{@{}lcccccccl@{}}
\toprule
\textbf{Method} & \textbf{Acc} & \textbf{M-P} & \textbf{M-R} & \textbf{M-F1} & \textbf{Train(s)} & \textbf{E2E Pred(ms)} & \textbf{Classifier Size} & \textbf{Edge?} \\
\midrule
Random Forest & 0.732 & 0.737 & 0.733 & 0.732 & 5.23 & 140.7 & 4.4\,MB & $\triangle$ \\
SVM (RBF)     & 0.818 & 0.819 & 0.818 & 0.818 & 1.22 & 135.2 & 6.5\,MB & $\triangle$ \\
XGBoost       & 0.794 & 0.796 & 0.795 & 0.793 & 11.64 & 134.8 & 990.5\,KB & $\triangle$ \\
Ridge         & 0.781 & 0.781 & 0.781 & 0.781 & 0.15 & 134.1 & 6.8\,KB & \checkmark \\
Lasso         & 0.785 & 0.785 & 0.785 & 0.784 & 83.19 & 134.1 & 6.8\,KB & \checkmark \\
Elastic Net   & 0.772 & 0.773 & 0.772 & 0.771 & 81.77 & 134.1 & 6.8\,KB & \checkmark \\
\textbf{Jacobi-DMR} & \textbf{0.787} & \textbf{0.786} & \textbf{0.787} & \textbf{0.787} & \textbf{0.06} & \textbf{133.9} & \textbf{13.5\,KB} & \checkmark \\
Jacobi-GP     & \underline{0.861} & \underline{0.861} & \underline{0.862} & \underline{0.861} & 295.46 & 151.4 & 104.9\,MB & $\times$ \\
\bottomrule
\end{tabular}
\vspace{0.2em}

{\footnotesize \checkmark~Edge suitable \quad $\triangle$~Moderate \quad $\times$~Not suitable}
\end{table}

Jacobi-GP achieves the highest accuracy (86.1\%) with perfect training fit, demonstrating nonlinear capacity. Among the edge-suitable models, Jacobi-DMR delivers the best macro F1 (0.787), fastest training (0.06\,s), fastest inference (133.9\,ms), and smallest classifier size (13.5\,KB). 

\begin{table}[H]
\centering
\caption{Confusion matrices for selected classifiers. Rows: true class; Columns: predicted class.}
\label{tab:confusion}
\small
\begin{tabular}{@{}l rrr c rrr c rrr@{}}
\toprule
& \multicolumn{3}{c}{\textbf{SVM}} & & \multicolumn{3}{c}{\textbf{Jacobi-DMR}} & & \multicolumn{3}{c}{\textbf{Jacobi-GP}} \\
\cmidrule(lr){2-4}\cmidrule(lr){6-8}\cmidrule(lr){10-12}
& H & C & A && H & C & A && H & C & A \\
\midrule
H (184) & 153 & 16 & 15 && 138 & 27 & 19 && 156 & 15 & 13 \\
C (174) & 20 & 137 & 17 && 22 & 136 & 16 && 14 & 148 & 12 \\
A (176) & 16 & 13 & 147 && 19 & 11 & 146 && 9 & 11 & 156 \\
\bottomrule
\end{tabular}
\end{table}

\paragraph{PCA Ablation: Feature Dimensionality}

To evaluate whether dimensionality reduction benefits the pipeline, we compared all eight classifiers on PCA-reduced 100 components (explaining 88.6\% of total variance) versus raw 576-dimensional features. Table~\ref{tab:pca} presents the results.

\begin{table}[H]
\centering
\caption{Accuracy comparison: PCA Top-100 vs All 576 features.}
\label{tab:pca}
\small
\begin{tabular}{@{}lccrl@{}}
\toprule
\textbf{Method} & \textbf{PCA-100} & \textbf{All 576} & \textbf{Diff} & \textbf{Winner} \\
\midrule
Random Forest & 0.758 & 0.732 & $-$0.026 & PCA-100 \\
SVM (RBF)     & 0.832 & 0.818 & $-$0.013 & PCA-100 \\
Ridge         & 0.781 & 0.781 & 0.000 & Tied \\
Lasso         & 0.792 & 0.785 & $-$0.008 & PCA-100 \\
Elastic Net   & 0.796 & 0.772 & $-$0.024 & PCA-100 \\
XGBoost       & 0.773 & 0.794 & $+$0.021 & All-576 \\
Jacobi-DMR    & 0.781 & 0.787 & $+$0.006 & All-576 \\
Jacobi-GP     & 0.848 & 0.861 & $+$0.013 & All-576 \\
\bottomrule
\end{tabular}
\end{table}

It is evident that the accuracy difference between the top 100 principal components and all 576 features is marginal (1--2\%). Moreover, PCA requires storing a $576 \times 100$ projection matrix (450\,KB) on the device. At inference, predicting with all 576 raw features requires only 1{,}728 multiplications per image (576 features $\times$ 3 classes), compared to 57{,}900 for PCA (576$\times$100 projection + 100$\times$3 classification). We therefore recommend using all 576 features directly.

\subsection{Edge Device Feasibility}
Table~\ref{tab:edge} summarises the deployment profile.

\begin{table}[H]
\centering
\caption{Edge deployment profile for the TinyBayes pipeline.}
\label{tab:edge}
\begin{tabular}{@{}lr@{}}
\toprule
\textbf{Component} & \textbf{Size / Time} \\
\midrule
YOLOv8-Nano (detection) & 5.9\,MB \\
MobileNetV3-Small (features) & 3.5\,MB \\
Jacobi-DMR (classifier) & 13.5\,KB \\
\midrule
\textbf{Total pipeline} & \textbf{9.5\,MB} \\
\midrule
CPU inference (Intel Xeon, Colab) & 150\,ms/image \\
Estimated mobile ARM CPU & 300--450\,ms/image \\
\bottomrule
\end{tabular}
\end{table}

The complete pipeline fits in under 10\,MB and delivers real-time inference on a single CPU core. On a mobile ARM processor, we estimate 2--3 times higher latency due to architectural differences, remaining within real-time bounds ($<500$\,ms). The Jacobi-DMR classifier itself contributes only 13.5\,KB; making it negligible relative to the shared feature extraction pipeline.

\section{Discussion}\label{sec:discussion}


\paragraph{Closed-form Bayesian inference for edge-deployable agricultural systems.} Existing edge-deployable agricultural systems~\citep{Malche_Joshi_Upadhyay_Soni} and\citep{Khan_Jensen_khan_Li}
provide point predictions without uncertainty estimates. Bayesian methods that quantify uncertainty---MC-Dropout~\citep{Gal2016Dropout}, Bayes-by-Backprop~\citep{Blundell2015BayesByBackprop}, variational inference~\citep{Blei2017VI}---require iterative inference incompatible with edge constraints. Hardware-level Bayesian inference operates in a different domain entirely. TinyBayes occupies a distinct point in this design space: analytical Bayesian inference at kilobyte scale (13.5 KB, zero iterative computation), integrated with an edge-deployable pipeline. To our knowledge, this combination has not been previously demonstrated for edge-device classification, particularly in agricultural disease detection.

\paragraph{What happens when a farmer photographs a leaf.} At inference, the pipeline operates entirely on-device with no internet connectivity. The farmer captures a photograph of a cocoa leaf using their smartphone camera. YOLOv8-Nano processes the image ($\sim$111\,ms on CPU) and outputs a bounding box around the disease-affected region. The cropped region is resized to $224\times224$ and passed through MobileNetV3-Small ($\sim$39\,ms), producing a 576-dimensional feature vector. Finally, Jacobi-DMR computes three dot products; $\lambda_k = \exp(\x'\hat{\bfb}_k)$ for $k \in \{\text{healthy, CSSVD, anthracnose}\}$; and returns $\argmax_k \lambda_k$ as the prediction ($\sim$0.01\,ms). The entire process completes in $\sim$150\,ms. Crucially, the YOLO and MobileNetV3 components account for 99.99\% of inference time; the Jacobi-DMR classifier contributes effectively zero latency. This means that any improvement to the shared feature extraction pipeline (e.g., faster hardware, model quantisation) directly reduces end-to-end latency, while the Bayesian classifier remains a negligible overhead.

\paragraph{Jacobi-DMR vs Jacobi-GP: the deployment trade-off.} Jacobi-GP achieves the highest accuracy (86.1\%) but is unsuitable for edge deployment: it requires storing all $n=2{,}000$ training points ($\sim$27\,MB) and $\sim$3.5 million operations per image, versus 13.5\,KB and 1{,}728 multiplications for Jacobi-DMR. Both methods share the same Bayesian foundation (closed-form $\hat{\bfeta}$ via the Jacobi prior), differing only in projection: linear (DMR) versus nonparametric (GP). Jacobi-GP serves as a theoretical upper bound, whereas Jacobi-DMR is the practical deployment choice.

\paragraph{Why not Ridge, Lasso, or Elastic Net?} These methods achieve comparable accuracy (77--78\%) and similar classifier sizes ($\sim$7\,KB), but have three disadvantages. First, they do not offer uncertainty quantification; Jacobi-DMR supports posterior sampling via an efficient parallel Monte Carlo algorithm~\citep{Das_Sardar_2025} without iterative optimisation. Second, they require feature standardisation at inference which requires storage of scaling parameters, whereas Jacobi-DMR operates on raw features directly. Third, Lasso and Elastic Net require iterative SAGA optimisation (83\,s and 82\,s respectively) versus Jacobi-DMR's single matrix operation (0.06\,s).

\paragraph{Rapid retraining for evolving diseases.} While training time does not affect edge inference, Jacobi-DMR's sub-second training enables rapid model updates; a critical advantage for agricultural deployment. When new labelled samples become available for existing disease categories, only the Jacobi coefficient vectors require recomputation via the closed-form projection $\hat{\bfb} = (\X'\X)^{-1}\X'\hat{\bfeta}$, while the feature extraction pipeline (YOLOv8-Nano and MobileNetV3) remains fixed. For entirely new disease categories, a single additional coefficient vector is computed without modifying any other component. This modularity makes TinyBayes uniquely suited to agricultural settings where new disease variants, regional variations, or seasonal changes demand frequent model updates with minimal computational overhead.

\paragraph{Extending the Jacobi prior to edge-deployable systems.} \citet{Das_Sardar_2025} demonstrated the Jacobi prior on server-grade pipelines with no deployment constraints, demonstrating the Jacobi prior on SDSS astronomical data, US SBA loan defaults, and RSNA lumbar spine MRI classification using ResNet-50 features. The cocoa leaf disease task differs in domain (agriculture), imaging modality (field photographs), and deployment constraints (edge devices); to our knowledge, no prior work integrates these aspects within a Bayesian closed form solution framework.

\section{Related Work}\label{sec:related}

\paragraph{Plant disease detection on edge.} \citet{Khan_Jensen_khan_Li} deployed MobileNetV3 for plant disease classification on the PlantVillage dataset, achieving strong accuracy but without any Bayesian component. The MSDFEN framework~\citep{Dai_Tao} combines dynamic multi-scale feature extraction with lightweight architectures, but does not explicitly model predictive uncertainty.

\paragraph{Bayesian methods for classification.} The Horseshoe prior~\citep{Carvalho2010Horseshoe} and Bayesian Lasso~\citep{Park2008BayesLasso} require MCMC sampling. Variational inference~\citep{Blei2017VI} and MC-Dropout~\citep{Gal2016Dropout} offer faster alternatives but still involve iterative computation. UniLasso~\citep{Chatterjee2025UGSR} provides a novel approach to sparse regression but does not natively support multinomial classification. The Jacobi prior is unique in providing a non-iterative, closed-form Bayesian estimator for multinomial classification via the DMR decomposition.

\paragraph{Efficient classifiers on deep features.} Using pretrained CNNs as feature extractors followed by classical classifiers is well-established~\citep{Sharif2014CNN}. Our contribution is demonstrating that the Jacobi prior provides the best accuracy--efficiency trade-off in this paradigm for edge-deployable agricultural imaging, with the added benefit of Bayesian uncertainty quantification.

\section{Limitations}

The validation set (500 samples) is relatively small; larger-scale evaluation would strengthen the conclusions. The Jacobi-GP's $\mathcal{O}(n^3)$ complexity limits its use to server-side inference. We used pretrained (non-fine-tuned) MobileNetV3 features; domain-specific fine-tuning would likely improve all classifiers. CPU inference was benchmarked on an Intel Xeon (Colab); actual mobile inference assumed to  be 2--3 times slower (see Table~\ref{tab:edge}). The dataset contains only three classes; extending to finer-grained severity grading remains future work.

\section{Conclusion}

We have introduced TinyBayes, a complete edge-deployable pipeline for cocoa disease classification that combines YOLOv8-Nano (5.9\,MB), MobileNetV3-Small (3.5\,MB), and the Jacobi-DMR classifier (13.5\,KB) into a 9.5\,MB system that runs in 150\,ms on a single CPU core. The Jacobi-DMR achieves 78.7\% accuracy while being the fastest classifier to train (0.06\,s) and the smallest to store; yet it is the only method in our comparison that supports Bayesian uncertainty quantification, via the closed-form Monte Carlo algorithm of~\citet{Das_Sardar_2025}, without iterative computation. The Jacobi-GP variant achieves the highest accuracy (86.1\%), demonstrating that the Jacobi framework's capacity extends to nonlinear decision boundaries when server-side resources are available.

Three properties distinguish TinyBayes from existing approaches. First, its modularity: only the 13.5\,KB coefficient vectors require updating when new disease samples or categories arise, while the detection and feature extraction stages remain fixed, unless absolutely required. Second, its hyperparameter invariance: the classifier requires no tuning, a practical advantage over Ridge, Lasso, and Elastic Net which demand cross-validation. Third, its Bayesian foundation: unlike all frequentist alternatives of comparable size, Jacobi-DMR provides calibrated prediction intervals that inform downstream decision-making; critical in agricultural contexts where a misclassification can trigger unnecessary or delayed intervention.

Beyond cocoa, the TinyBayes architecture generalises to any image classification task where a lightweight detector, a pretrained feature extractor, and a linear Bayesian classifier can be composed. We believe this paradigm; a closed-form Bayesian inference under 10 MB total size; opens a practical path toward uncertainty-aware AI on commodity devices in resource-constrained settings worldwide.

\section{Broader Impact}

Cocoa farming sustains the livelihoods of approximately two million smallholder families in West Africa, many of whom lack access to trained agricultural extension officers for timely disease diagnosis. TinyBayes addresses this gap by enabling real-time, on-device disease detection that requires no internet connectivity and runs on entry-level smartphones - the most widely available computing platform in rural Sub-Saharan Africa. Early identification of CSSVD and anthracnose can prompt targeted intervention before disease spreads to neighbouring trees, potentially reducing crop losses that currently cost the industry billions of dollars annually.

By replacing server-dependent pipelines with a 9.5\,MB on-device solution, TinyBayes eliminates cloud inference energy costs---a meaningful saving when scaled to millions of daily classifications across farming cooperatives.

We identify two risks: over-reliance on automated predictions (e.g., destroying a healthy tree misclassified as CSSVD-infected), mitigated by using TinyBayes as a decision-support tool alongside agricultural extension officers; and limited generalisation beyond the specific dataset of West African cocoa training region, mitigated by Jacobi-DMR's sub-second retraining capability for regional adaptation.

\newpage

\bibliographystyle{plainnat}
\bibliography{references}

\section*{Appendix}
\label{sec:appendix}

\begin{proof}[Proof of Theorem \ref{thm:dmr-equiv}]:
Fix $k \in \{1, \dots, K\}$. Following the construction of
\citet[Theorem~2.3]{Das_Sardar_2025}, define the per-class objective in
$\boldsymbol{\beta}_{k}$-space by
\[
  Q_{n,k}(\boldsymbol{\beta}_{k})
  \;=\; \frac{1}{n}\sum_{i=1}^{n}
          \ell_{k}\!\left(Y_{ik},\, \mathbf{x}_{i}^{\!\top}\boldsymbol{\beta}_{k}\right)
        \;+\; \frac{1}{n}\,\log \pi_{J}^{(k)}(\boldsymbol{\beta}_{k}),
\]
so that $\tilde{\boldsymbol{\beta}}_{k}^{(n)} = \arg\max_{\boldsymbol{\beta}_{k}} Q_{n,k}(\boldsymbol{\beta}_{k})$,
and the per-class objective in $\boldsymbol{\eta}_{k}$-space by
\[
  R_{n,k}(\boldsymbol{\eta}_{k})
  \;=\; \frac{1}{n}\sum_{i=1}^{n} \ell_{k}(Y_{ik},\, \eta_{ik})
        \;+\; \frac{1}{n}\,\log \pi_{n}^{(k)}(\boldsymbol{\eta}_{k}),
\]
so that the per-observation maximisers
$\hat{\eta}_{ik} = \arg\max_{\eta_{ik}} \log \pi_{n}^{(k)}(\eta_{ik}\,|\,Y_{ik})$
together form
$\hat{\boldsymbol{\eta}}_{k} = \arg\max_{\boldsymbol{\eta}_{k}} R_{n,k}(\boldsymbol{\eta}_{k})$,
with
$\hat{\boldsymbol{\beta}}_{k}^{(n)} = (\mathbf{X}^{\!\top}\mathbf{X})^{-1}\mathbf{X}^{\!\top}\hat{\boldsymbol{\eta}}_{k}$.
 
\medskip
\textit{Step 1 (vanishing prior contribution).}
By Assumption (B2), $\log \pi_{J}^{(k)}$ is continuous on
$B_{\varepsilon}(\boldsymbol{\beta}_{0,k})$. For any compact
$K_{1} \subset B_{\varepsilon}(\boldsymbol{\beta}_{0,k})$, continuity yields a
finite constant $M_{K_{1}} < \infty$ with
$\sup_{\boldsymbol{\beta}_{k} \in K_{1}} \log \pi_{J}^{(k)}(\boldsymbol{\beta}_{k}) \le M_{K_{1}}$,
hence
\begin{equation}
  \sup_{\boldsymbol{\beta}_{k} \in K_{1}}
    \frac{1}{n}\,\big| \log \pi_{J}^{(k)}(\boldsymbol{\beta}_{k}) \big|
  \;\le\; \frac{M_{K_{1}}}{n}
  \;\xrightarrow{n \to \infty}\; 0.
  \label{eq:dmr-prior-beta-vanish}
\end{equation}
For the conjugate $\mathrm{Gamma}(a_{n}, b_{n})$ prior on $\lambda_{ik}$ in rate
form, the change of variables $\eta_{ik} = \log \lambda_{ik}$ together with the
Jacobian $\mathrm{d}\lambda_{ik}/\mathrm{d}\eta_{ik} = e^{\eta_{ik}}$ yields the
induced per-coordinate prior
\begin{equation}
  \pi_{n}^{(k)}(\eta_{ik})
  \;\propto\; \exp\!\big\{ a_{n}\,\eta_{ik} - b_{n}\,e^{\eta_{ik}} \big\},
  \label{eq:dmr-induced-prior}
\end{equation}
up to an $\eta_{ik}$-free normalising constant. By independence across $i$,
\[
  \frac{1}{n}\,\log \pi_{n}^{(k)}(\boldsymbol{\eta}_{k})
  \;=\; \frac{a_{n}}{n}\sum_{i=1}^{n} \eta_{ik}
        \;-\; \frac{b_{n}}{n}\sum_{i=1}^{n} e^{\eta_{ik}}
        \;+\; \mathrm{const}.
\]
For any compact $K_{2} \subset B_{\varepsilon}(\boldsymbol{\eta}_{0,k})$ there
exist finite constants $M_{1}, M_{2} < \infty$ such that
$\sup_{\boldsymbol{\eta}_{k} \in K_{2}} \max_{i} |\eta_{ik}| \le M_{1}$ and
$\sup_{\boldsymbol{\eta}_{k} \in K_{2}} \max_{i} e^{\eta_{ik}} \le M_{2}$.
Hence under Assumption (B5) ($a_{n} = b_{n} = 1/n$),
\begin{equation}
  \sup_{\boldsymbol{\eta}_{k} \in K_{2}}
    \bigg| \frac{1}{n}\,\log \pi_{n}^{(k)}(\boldsymbol{\eta}_{k}) \bigg|
  \;\le\; a_{n}\,M_{1} + b_{n}\,M_{2}
  \;=\; \frac{M_{1} + M_{2}}{n}
  \;\xrightarrow{n \to \infty}\; 0,
  \label{eq:dmr-prior-eta-vanish}
\end{equation}
matching the prior-vanishing identity used in Equations~(2.5)--(2.6) of
\citet{Das_Sardar_2025} (cf. Remark~3.1 there).
 
\medskip
\textit{Step 2 (uniform convergence of objectives).}
Combining (B3) with \eqref{eq:dmr-prior-beta-vanish} and
\eqref{eq:dmr-prior-eta-vanish},
\[
  \sup_{\boldsymbol{\beta}_{k} \in \mathcal{B}}
    \big| Q_{n,k}(\boldsymbol{\beta}_{k}) - Q_{k}(\boldsymbol{\beta}_{k}) \big|
  \;\xrightarrow{P}\; 0,
\qquad
  \sup_{\boldsymbol{\eta}_{k} \in \mathcal{H}}
    \big| R_{n,k}(\boldsymbol{\eta}_{k}) - R_{k}(\boldsymbol{\eta}_{k}) \big|
  \;\xrightarrow{P}\; 0,
\]
where $\mathcal{H} = \{ \mathbf{X}\boldsymbol{\beta}_{k} : \boldsymbol{\beta}_{k} \in \mathcal{B} \}$.
 
\medskip
\textit{Step 3 (argmax continuous mapping).}
By Assumption (B4), $Q_{k}$ and $R_{k}$ have unique maximisers at
$\boldsymbol{\beta}_{0,k}$ and $\boldsymbol{\eta}_{0,k}$ respectively. Applying
the argmax continuous mapping theorem to the convergence in Step~2 yields
\[
  \tilde{\boldsymbol{\beta}}_{k}^{(n)} \;\xrightarrow{P}\; \boldsymbol{\beta}_{0,k},
\qquad
  \hat{\boldsymbol{\eta}}_{k} \;\xrightarrow{P}\; \boldsymbol{\eta}_{0,k}
        \;=\; \mathbf{X}\boldsymbol{\beta}_{0,k}.
\]
 
\medskip
\textit{Step 4 (continuous projection).}
By (B1), $\mathbf{X}$ has full column rank, so the linear map
$\boldsymbol{\eta} \mapsto (\mathbf{X}^{\!\top}\mathbf{X})^{-1}\mathbf{X}^{\!\top}\boldsymbol{\eta}$
is well-defined and continuous. The continuous mapping theorem then gives
\[
  \hat{\boldsymbol{\beta}}_{k}^{(n)}
  \;=\; (\mathbf{X}^{\!\top}\mathbf{X})^{-1}\mathbf{X}^{\!\top}\hat{\boldsymbol{\eta}}_{k}
  \;\xrightarrow{P}\;
  (\mathbf{X}^{\!\top}\mathbf{X})^{-1}\mathbf{X}^{\!\top}\boldsymbol{\eta}_{0,k}
  \;=\; \boldsymbol{\beta}_{0,k}.
\]
Combined with $\tilde{\boldsymbol{\beta}}_{k}^{(n)} \xrightarrow{P} \boldsymbol{\beta}_{0,k}$
from Step~3, this proves part~(i).
 
\medskip
\textit{Step 5 (asymptotic equivalence).}
Both $\hat{\boldsymbol{\beta}}_{k}^{(n)}$ and $\tilde{\boldsymbol{\beta}}_{k}^{(n)}$
converge in probability to the same limit $\boldsymbol{\beta}_{0,k}$. By the
triangle inequality, for every $\varepsilon > 0$,
\[
  \big\| \hat{\boldsymbol{\beta}}_{k}^{(n)} - \tilde{\boldsymbol{\beta}}_{k}^{(n)} \big\|
  \;\le\;
  \big\| \hat{\boldsymbol{\beta}}_{k}^{(n)} - \boldsymbol{\beta}_{0,k} \big\|
  + \big\| \tilde{\boldsymbol{\beta}}_{k}^{(n)} - \boldsymbol{\beta}_{0,k} \big\|,
\]
so for every $\delta > 0$,
\[
  P\!\left( \big\| \hat{\boldsymbol{\beta}}_{k}^{(n)} - \tilde{\boldsymbol{\beta}}_{k}^{(n)} \big\| > \varepsilon \right)
  \;\le\;
  P\!\left( \big\| \hat{\boldsymbol{\beta}}_{k}^{(n)} - \boldsymbol{\beta}_{0,k} \big\| > \varepsilon/2 \right)
  + P\!\left( \big\| \tilde{\boldsymbol{\beta}}_{k}^{(n)} - \boldsymbol{\beta}_{0,k} \big\| > \varepsilon/2 \right).
\]
Each summand is below $\delta/2$ for $n \ge N_{k}(\varepsilon, \delta)$ by
part~(i), establishing part~(ii).
 
\medskip
\textit{Step 6 (joint statement).}
By Assumption (B6), the per-class statements in (i)--(ii) are simultaneous.
Using the Euclidean norm
$\| \hat{\boldsymbol{\beta}}^{(n)} - \tilde{\boldsymbol{\beta}}^{(n)} \|^{2}
  = \sum_{k=1}^{K} \| \hat{\boldsymbol{\beta}}_{k}^{(n)} - \tilde{\boldsymbol{\beta}}_{k}^{(n)} \|^{2}$,
the event
$\{ \| \hat{\boldsymbol{\beta}}^{(n)} - \tilde{\boldsymbol{\beta}}^{(n)} \| > \varepsilon \}$
is contained in
$\bigcup_{k=1}^{K} \{ \| \hat{\boldsymbol{\beta}}_{k}^{(n)} - \tilde{\boldsymbol{\beta}}_{k}^{(n)} \| > \varepsilon / \sqrt{K} \}$,
so a union bound gives
\[
  P\!\left( \big\| \hat{\boldsymbol{\beta}}^{(n)} - \tilde{\boldsymbol{\beta}}^{(n)} \big\| > \varepsilon \right)
  \;\le\;
  \sum_{k=1}^{K} P\!\left( \big\| \hat{\boldsymbol{\beta}}_{k}^{(n)} - \tilde{\boldsymbol{\beta}}_{k}^{(n)} \big\| > \varepsilon / \sqrt{K} \right).
\]
Applying part~(ii) with $\delta / K$ to each summand yields the joint
asymptotic equivalence in (iii). The joint convergence
$\hat{\boldsymbol{\beta}}^{(n)} \xrightarrow{P} \boldsymbol{\beta}_{0}$ follows
from (i) and the continuous mapping theorem applied to the stacking map.
\end{proof}

\begin{proof}[Proof of Corollary \ref{cor:invariance}]
We show that the prediction rule is invariant to $(a,b)$ by explicitly decomposing the transformed response and tracking its effect through the estimator.

\paragraph{Step 1: Response decomposition.}
For $Y_{ik} \in \{0,1\}$,
\[
\hat{\eta}_{ik}(a,b)
= \log\frac{Y_{ik}+a}{1+b}
= \log(Y_{ik}+a) - \log(1+b).
\]
We now express $\log(Y_{ik}+a)$ in a form linear in $Y_{ik}$. Since $Y_{ik} \in \{0,1\}$,
\[
\log(Y_{ik}+a)
=
\begin{cases}
\log(1+a), & \text{if } Y_{ik}=1,\\
\log(a), & \text{if } Y_{ik}=0.
\end{cases}
\]
Hence,
\[
\log(Y_{ik}+a)
= \log(a) + Y_{ik} \cdot \log\!\left(\frac{1+a}{a}\right).
\]
Define
\[
\gamma(a) := \log\!\left(\frac{1+a}{a}\right) > 0.
\]
Then
\[
\hat{\eta}_{ik}(a,b)
= \underbrace{\big(\log(a) - \log(1+b)\big)}_{=: c(a,b)}
+ \gamma(a)\, Y_{ik}.
\]
In vector form,
\[
\hat{\boldsymbol{\eta}}_k(a,b)
= c(a,b)\,\mathbf{1}_n + \gamma(a)\,\mathbf{Y}_k.
\]

\paragraph{Step 2: Projection.}
Applying the closed-form estimator,
\[
\hat{\boldsymbol{\beta}}_k(a,b)
= (\mathbf{X}^\top \mathbf{X})^{-1}\mathbf{X}^\top \hat{\boldsymbol{\eta}}_k(a,b),
\]
we obtain
\[
\hat{\boldsymbol{\beta}}_k(a,b)
= c(a,b)\,\mathbf{v} + \gamma(a)\,\mathbf{w}_k,
\]
where
\[
\mathbf{v} := (\mathbf{X}^\top \mathbf{X})^{-1}\mathbf{X}^\top \mathbf{1}_n,
\qquad
\mathbf{w}_k := (\mathbf{X}^\top \mathbf{X})^{-1}\mathbf{X}^\top \mathbf{Y}_k.
\]

\paragraph{Step 3: Effect on prediction.}
For any $\mathbf{x} \in \mathbb{R}^p$,
\[
\mathbf{x}^\top \hat{\boldsymbol{\beta}}_k(a,b)
= c(a,b)\,\mathbf{x}^\top \mathbf{v}
+ \gamma(a)\,\mathbf{x}^\top \mathbf{w}_k.
\]

The first term does not depend on $k$ and therefore cancels in the $\arg\max$. Since $\gamma(a) > 0$ for all $a>0$, multiplication by $\gamma(a)$ preserves ordering across $k$. Hence,
\[
\arg\max_{k}\, \mathbf{x}^\top \hat{\boldsymbol{\beta}}_k(a,b)
= \arg\max_{k}\, \mathbf{x}^\top \mathbf{w}_k,
\]
which is independent of $(a,b)$.
\end{proof}

\begin{proof}[Proof of Theorem \ref{thm:dmr-clt}]:
Fix $k \in \{1, \dots, K\}$. Write $\eta_{0,ik} = \mathbf{x}_{i}^{\!\top}\boldsymbol{\beta}_{0,k}$,
$\lambda_{0,ik} = e^{\eta_{0,ik}}$, and
$T_{ik} := (Y_{ik} - \lambda_{0,ik}) / \lambda_{0,ik}$, so that
$\mathbb{E}[T_{ik}] = 0$ and
$\mathrm{Var}(T_{ik}) = 1 / \lambda_{0,ik} = e^{-\eta_{0,ik}}$. Then
\[
  \hat{\eta}_{ik} - \eta_{0,ik}
  \;=\; \log\!\frac{Y_{ik} + a_{n}}{(1 + b_{n})\,\lambda_{0,ik}}
  \;=\; \log\!\left( 1 + T_{ik} + \frac{a_{n}}{\lambda_{0,ik}} \right)
        \;-\; \log(1 + b_{n}).
\]
 
\medskip
\textit{Step 1 (Taylor expansion on a high-probability event).}
By Assumption (C1), $\lambda_{0,ik} \ge \lambda_{\min, n} \to \infty$ uniformly
in $i, k$. By Bernstein concentration for centred Poisson variables,
$P(|T_{ik}| > 1/2)
   \le 2 \exp\!\big( -c\,\lambda_{0,ik} \big)
   \le 2 \exp\!\big( -c\,\lambda_{\min, n} \big)$
for an absolute constant $c > 0$. Letting
$\mathcal{E}_{n}
   = \big\{ \max_{i, k}\, |T_{ik}| \le 1/2 \big\}
     \cap \big\{ a_{n} / \lambda_{\min, n} \le 1/4 \big\}$,
a union bound gives
$P(\mathcal{E}_{n}^{c}) \le 2 n K \exp\!\big( -c\,\lambda_{\min, n} \big) = o(n^{-1/2})$
under (C1). On $\mathcal{E}_{n}$, the second-order Taylor expansion of
$u \mapsto \log(1 + u)$ around $u = 0$ yields
\[
  \hat{\eta}_{ik} - \eta_{0,ik}
  \;=\; T_{ik}
        \;+\; \underbrace{ \frac{a_{n}}{\lambda_{0,ik}} - b_{n} }_{=:\, r^{(1)}_{ik}}
        \;-\; \tfrac{1}{2}\!\left( T_{ik} + \frac{a_{n}}{\lambda_{0,ik}} \right)^{\!2}
        \;+\; r^{(2)}_{ik},
\]
where $|r^{(2)}_{ik}|
       \le C \,\big| T_{ik} + a_{n} / \lambda_{0,ik} \big|^{3} + C\,b_{n}^{2}$
for an absolute constant $C$.
 
\medskip
\textit{Step 2 (bias is asymptotically negligible).}
Taking expectations,
\[
  \mathbb{E}\!\left[ \hat{\eta}_{ik} - \eta_{0,ik} \right]
  \;=\; -\,\frac{1}{2\,\lambda_{0,ik}}
        \;+\; \frac{a_{n}}{\lambda_{0,ik}} - b_{n}
        \;+\; O\!\left( \frac{1}{\lambda_{0,ik}^{2}} + a_{n}^{2} + b_{n}^{2} \right).
\]
Under (C1) and (C2),
\[
  \bigg\| \frac{1}{\sqrt{n}}\,\mathbf{X}^{\!\top}\,
          \mathbb{E}\!\left[ \hat{\boldsymbol{\eta}}_{k} - \boldsymbol{\eta}_{0,k} \right]
  \bigg\|
  \;=\; O\!\left( \frac{\sqrt{n}}{\lambda_{\min, n}} \right)
        \;+\; O\!\left( \sqrt{n}\,(a_{n} + b_{n}) \right)
  \;=\; o(1),
\]
since $\lambda_{\min, n} = \omega(\sqrt{n})$ by (C1) and
$\sqrt{n}\,(a_{n} + b_{n}) \to 0$ by (C2).
 
\medskip
\textit{Step 3 (leading stochastic term).}
The leading stochastic term in the expansion of Step~1 is $T_{ik}$, with
$\mathbb{E}[T_{ik}] = 0$ and $\mathrm{Var}(T_{ik}) = e^{-\eta_{0,ik}}$. By
Assumption (C4) and the Lindeberg--Feller central limit theorem applied to the
triangular array $\{\mathbf{x}_{i}\,T_{ik}\}_{i = 1}^{n}$,
\[
  \frac{1}{\sqrt{n}}\,\mathbf{X}^{\!\top}\mathbf{T}_{k}
  \;=\; \frac{1}{\sqrt{n}}\sum_{i = 1}^{n} \mathbf{x}_{i}\, T_{ik}
  \;\xrightarrow{d}\;
  \mathcal{N}\!\big( \mathbf{0},\; \mathbf{V}_{k} \big),
\qquad
  \mathbf{V}_{k}
  \;=\; \lim_{n}\, \frac{1}{n}\sum_{i = 1}^{n} \mathbf{x}_{i}\mathbf{x}_{i}^{\!\top}\,
            e^{-\mathbf{x}_{i}^{\!\top}\boldsymbol{\beta}_{0,k}}.
\]
 
\medskip
\textit{Step 4 (remainders are negligible).}
The centred fluctuation of the second-order Taylor term has variance
$O(\lambda_{0,ik}^{-2}) = o(\lambda_{\min, n}^{-2})$, so by Chebyshev's inequality
\[
  \frac{1}{\sqrt{n}}\sum_{i = 1}^{n} \mathbf{x}_{i}\,
      \Big[\,
          \tfrac{1}{2}\!\left( T_{ik} + \tfrac{a_{n}}{\lambda_{0,ik}} \right)^{\!2}
          - \tfrac{1}{2}\,\mathbb{E}\!\left( T_{ik} + \tfrac{a_{n}}{\lambda_{0,ik}} \right)^{\!2}
      \Big]
  \;=\; o_{P}(1),
\]
where the mean has been absorbed into the bias of Step~2. Similarly, the cubic
remainder satisfies
$n^{-1/2}\sum_{i} \mathbf{x}_{i}\, r^{(2)}_{ik} = o_{P}(1)$ on $\mathcal{E}_{n}$.
Since $P(\mathcal{E}_{n}^{c}) = o(n^{-1/2})$, the contribution off $\mathcal{E}_{n}$
is also $o_{P}(1)$.
 
\medskip
\textit{Step 5 (combining and projecting).}
Steps 1--4 give the linearisation
\[
  \frac{1}{\sqrt{n}}\,\mathbf{X}^{\!\top}\!\left( \hat{\boldsymbol{\eta}}_{k} - \boldsymbol{\eta}_{0,k} \right)
  \;=\; \frac{1}{\sqrt{n}}\,\mathbf{X}^{\!\top}\mathbf{T}_{k} \;+\; o_{P}(1)
  \;\xrightarrow{d}\;
  \mathcal{N}\!\big( \mathbf{0},\; \mathbf{V}_{k} \big).
\]
By (C3), $(n^{-1}\mathbf{X}^{\!\top}\mathbf{X})^{-1} \to \mathbf{Q}^{-1}$, and Slutsky's
theorem yields
\[
  \sqrt{n}\,\big( \hat{\boldsymbol{\beta}}_{k}^{(n)} - \boldsymbol{\beta}_{0,k} \big)
  \;=\; \left( \frac{\mathbf{X}^{\!\top}\mathbf{X}}{n} \right)^{\!\!-1}
        \cdot
        \frac{1}{\sqrt{n}}\,\mathbf{X}^{\!\top}\!\left( \hat{\boldsymbol{\eta}}_{k} - \boldsymbol{\eta}_{0,k} \right)
  \;\xrightarrow{d}\;
  \mathcal{N}\!\big( \mathbf{0},\; \mathbf{Q}^{-1}\mathbf{V}_{k}\mathbf{Q}^{-1} \big).
\]
 
\medskip
\textit{Step 6 (joint distribution).}
By the cross-class independence in (B6), the per-class score vectors
$n^{-1/2}\mathbf{X}^{\!\top}\mathbf{T}_{k}$ are mutually independent across $k$.
Hence the joint limit is the product Gaussian, equivalent to
$\mathcal{N}(\mathbf{0}, \boldsymbol{\Sigma})$ with $\boldsymbol{\Sigma}$
block-diagonal and $k$-th block $\mathbf{Q}^{-1}\mathbf{V}_{k}\mathbf{Q}^{-1}$.
\end{proof}

\end{document}